\documentclass[letterpaper, 10 pt, conference]{ieeeconf}  % Comment this line out if you need a4paper

\IEEEoverridecommandlockouts                              % This command is only needed if 
\usepackage{graphicx}      % include this line if your document contains figures
\graphicspath{{figures/}}
\usepackage{float,amsmath,amsfonts,xcolor,graphicx}
\usepackage{booktabs}
\usepackage{multirow}
\usepackage{dsfont}
\usepackage{cite}
\usepackage[ruled,vlined,linesnumbered]{algorithm2e}
\usepackage{algpseudocode}
\usepackage{lipsum}

\usepackage[english]{babel}

\newtheorem{problem}{Problem}[section]
\newtheorem{theorem}{Theorem}[section]
\newtheorem{example}{Example}[section]

\newtheorem{remark}{Remark}[section]
\newtheorem{assumption}{Assumption}[section]

\newtheorem{definition}{Definition}[section]
\usepackage{xcolor}

\definecolor{wheat}{rgb}{0.96,0.87,0.70}

\DeclareMathOperator*{\argmin}{arg\,min}

\newif\ifconfidential

 \confidentialtrue
 
 \ifconfidential
\usepackage{fancyhdr}
\fi
\graphicspath{{Figures/}}
\title{\LARGE \bf
Adaptive Sampling-based Motion Planning with Control Barrier Functions
}

\author{Ahmad Ahmad, Calin Belta, and Roberto Tron}

\begin{document}

\maketitle
\thispagestyle{empty}
\pagestyle{empty}

%%%%%%%%%%%%%%%%%%%%%%%%%%%%%%%%%%%%%%%%%%%%%%%%%%%%%%%%%%%%%%%%%%%%%%%%%%%%%%%%
\begin{abstract}

Sampling-based algorithms, such as Rapidly Exploring Random Trees (RRT) and its variants, have been used extensively for motion planning. 
Control barrier functions (CBFs) have been recently proposed to synthesize controllers for safety-critical systems. In this paper, we combine the effectiveness of RRT-based algorithms with the safety guarantees provided by CBFs in a method called CBF-RRT$^\ast$. CBFs are used for local trajectory planning for RRT$^\ast$, avoiding explicit collision checking of the extended paths. We prove that CBF-RRT$^\ast$ preserves the probabilistic completeness of RRT$^\ast$. Furthermore, in order to improve the sampling efficiency of the algorithm, we equip the algorithm with an adaptive sampling procedure, which is based on the cross-entropy method (CEM) for importance sampling (IS). The procedure exploits the tree of samples to focus the sampling in promising regions of the configuration space. We demonstrate the efficacy of the proposed algorithms through simulation examples.

\end{abstract}

%%%%%%%%%%%%%%%%%%%%%%%%%%%%%%%%%%%%%%%%%%%%%%%%%%%%%%%%%%%%%%%%%%%%%%%%%%%%%%%%
\section{Introduction}
Many state-of-the-art single query motion planning algorithms rely on randomized sampling to explore the configuration space, and build a path from a starting point to a goal region incrementally. Such algorithms are appealing for two main reasons. First, they avoid building the configuration space explicitly, which might be challenging in high-dimensional spaces. Rather, in the search for a path to the goal, they generate exploration paths and check if they do not coincide with obstacles. Second, given their exploratory nature and the fact that paths are typically built incrementally, one can impose differential constraints on the exploratory samples to generate paths that are dynamically feasible.

Rapidly-exploring random trees (RRT) \cite{lavalle1998rapidly} and its variants (see, e.g., \cite{LaValle_RRTconnect,Branicky2003,lavalle2001randomized,Yang2019h}) are sampling-based motion planning algorithms are simple to implement and are probabilistically complete \cite{RRTcompleteness}. RRTs aim to rapidly explore the configuration space and build a tree rooted at a starting configuration to find a path to a goal region. Karaman and Frazzoli \cite{KaramanRRTstarIJRR} proposed RRT$^\ast$, where each newly added vertex to the RRT tree is rewired with a possible better connection, for which the cost to reach the rewired vertex from the root vertex is reduced. This approach makes a path found with RRT$^\ast$ asymptotically optimal \cite{rev_asymptoticRRTstar}. Given its success in motion planning, in the past decade, there has been a  large number of research efforts to improve RRT$^\ast$ sampling. Examples include informed-RRT$^\ast$ \cite{Gammell2014} and its variant batch-informed-RRT$^\ast$ \cite{Gammell2015}, which construct an informed elliptical sampling region that shrinks as the length of the path decreases, which leads to faster convergence to the optimal path. 
Kobilarov \cite{Kobilarov2012e} introduced CE-RRT$^\ast$, which uses the cross-entropy method (CEM) \cite{Rubinstein1999} for importance sampling (IS). 

The works in \cite{lavalle2001randomized,Karaman2010a} impose differential constraints on the vertices of RRT and RRT$^\ast$, respectively, to produce feasible paths according to the robot kinodynamics. Sampling in informed spaces is generalized in \cite{Kino_informedRRTstar} to produce dynamically feasible paths, which are then used with informed-RRT$^\ast$. Recently, Wu \textit{et al.} \cite{Wu2020e} developed rapidly-exploring random reachable set trees (R3T), which constrain the expansion of RRT (and RRT$^\ast$) trees to be in the vertices' approximated reachable sets, which helps with finding dynamically feasible paths using fewer iterations than naively attempting to steering a vertex to a sampled configuration. Recent developments in controlling safety critical systems using control barrier functions (CBF) \cite{Ames2019,Ames2014,Ames2014a,Ames2017b,CBFmobile}, are exploited by Yang \textit{et al.} in CBF-RRT \cite{Yang2019a}. In this work, the authors model a safe set that contains the collision-free configurations of the robot, which is then used with a CBF-based controller to generate inputs that expand the tree in the safe set. In \cite{cbfrrt_app_Fainekos_ppr}, CBF-RRT is used to generate safe motion trajectories to safely navigate in environments with moving humans.

In this work, we develop two variants of RRT$^\ast$ in which we generate feasible trajectories using CBF-based controller syntheses, and we aim to improve the sampling of the algorithm. The contributions of the proposed work are as follows. First, CBF-RRT$^\ast$ (\S \ref{sec:CBF-RRT*}), an RRT$^\ast$ variant that is equipped with two local motion planners that generate CBF-based control inputs for exploring and expanding the RRT$^\ast$ tree (\S \ref{subbsec:Exp_LMP}), and for steering to exact desired configuration when rewiring a vertex (\S \ref{subbsec:Ext_LMP}). Using these local planners, we avoid explicit collision-checking procedure, where the trajectories are guaranteed to be in a safety set (a set of configurations that are collision-free). Second, Adaptive CBF-RRT$^\ast$ (\S \ref{sec:AdapCBF-RRT*}), a variant in which we exploit the exploration tree to focus the sampling in promising regions. To do so, we incorporate the algorithm with adaptive sampling procedure using the cross-entropy method (CEM) with nonparametric density estimation (\S \ref{subsec:CEM}). 

Figure \ref{fig:illus_CE_CBF-RRT*} illustrates the proposed adaptive iterations of CBF-RRT$^\ast$ sampling. At each newly added vertex, we attempt to steer the vertex to the goal region (left picture). Amongst the succeeded attempts, an elite subset of the produced trajectories is considered to adapt the sampling distribution (middle picture); sparsified configurations of the elite set (right picture) are then used to estimate a sampling density function (SDF) for the following sampling iteration. The proposed work is validated through simulation example in \S \ref{sec:sim_exps}, which shows that the proposed variant converges to the solution faster than the RRT$^\ast$.

\begin{figure}[htb]
	\begin{center}
		\includegraphics[width=.8\linewidth]{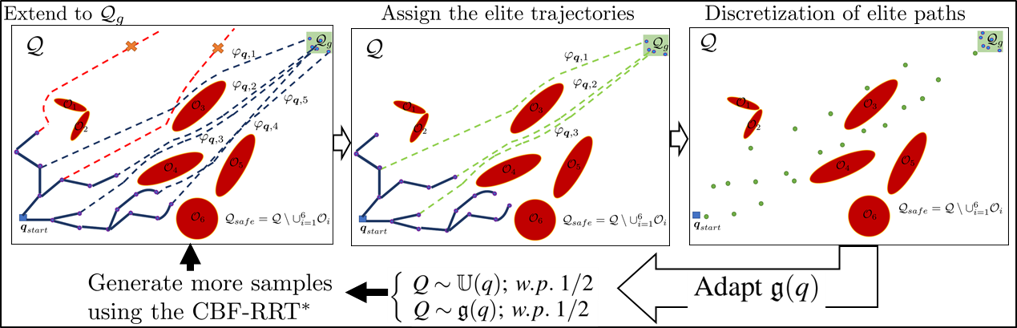}
		\medskip
	\end{center}\hfill
	\caption{Illustration of the proposed adaptive sampling using CEM with CBF-RRT$^\ast$}\label{fig:illus_CE_CBF-RRT*}
\end{figure}
%%%%%%%%%%%%%%%
%%%% SECTION %%
%%%%%%%%%%%%%%%
\section{Problem Formulation and Approach}\label{sec:prblm_frmthion}
Consider a robot with a configuration 
$\boldsymbol{q}\in\mathcal{Q}\subset\mathbb{R}^d$, where $\mathcal{Q}$ is the configuration space and $\mathbb{R}^d$ is the $d$-dimensional Euclidean space. Let the robot dynamics be modeled as the following nonlinear affine control dynamics,
\begin{equation}\label{eq:aff_system}
	\begin{aligned}
		\dot{\boldsymbol{q}} = f(\boldsymbol{q})+g(\boldsymbol{q})\boldsymbol{u},
	\end{aligned}
\end{equation}
where $\boldsymbol{u}\in\mathcal{U}\subset\mathbb{R}^m$ is the control input, $\mathcal{U}$ is the allowable control set, and $f:\mathbb{R}^d\rightarrow\mathbb{R}^d$ and $g:\mathbb{R}^d\rightarrow \mathbb{R}^{d\times m}
$ are assumed to be locally Lipschitz functions. 

Obstacle $i$, $i=1,\ldots,n$, is denoted by $\mathcal{O}_i\subset\mathcal{Q}$ \footnote{We present the obstacles in the workspace directly as their image in $\mathcal{Q}$, i.e., robot's configurations that cause it to collide with obstacles.}. The obstacle-free configuration space, which we denote it as the safe configuration space, is given by $\mathcal{Q}_{safe}= \mathcal{Q}\setminus\bigcup\limits_{i=1}^{n}\mathcal{O}_i$.

Following \cite{Kobilarov2012e}, for a time horizon $T\in\mathbb{R}_{>0}$, let $\varphi:[0,T]\times\mathbb{R}_{>0}\to\mathcal{U}\times\mathcal{Q}$, $\varphi(t,T):=(\boldsymbol{u}(t),\boldsymbol{q}(t))$
be pairs of a control input $\boldsymbol{u}(t)\in\mathcal{U},\;\forall t\in[0,T]$ and the produced
trajectory $\boldsymbol{q}(t)\in\mathcal{Q}$ that satisfies system (\ref{eq:aff_system}). 

For a given starting configuration $\boldsymbol{q}_{start}\in\mathcal{Q}_{safe}$ and a goal region $\mathcal{Q}_{goal}\subset\mathcal{Q}_{safe}$, we define the set $\mathcal{G}$ as the set of control inputs and the produced trajectory pairs, for which the trajectory to be in $\mathcal{Q}_{safe}$, starts at $\boldsymbol{q}_{start}$ and fall in $\mathcal{Q}_{goal}$. I.e., $\mathcal{G}:=\{\varphi(t,T)\,|\;\boldsymbol{q}(0) = \boldsymbol{q}_{start},\;\boldsymbol{q}_{T}\in\mathcal{Q}_{goal},\;\boldsymbol{q}(t)\in\mathcal{Q}_{safe}, \text{ system }(\ref{eq:aff_system}), \forall t\in[0,T], T\in\mathbb{R}_{>0}\}$. The cost functional of $\varphi\in\mathcal{G}$ is defined as $J(\varphi) \,:=\, \int_{0}^{T}C(\varphi(t,T))dt$, where $C:\mathcal{U}\times\mathcal{Q}\to\mathbb{R}_{>0}$ is the running cost.

% \Ahmad{Where $\mathcal{G}$ is used and why it is needed: - OMPP formulation where search upon this set - in $\texttt{sample}$ feasible trajectories to adapt the sampling distribution from (consider the case of sampling in the input space as oppose the configuration space (add this to the future work discussion)) - in $\texttt{extToGoal}$ which is}
% \Ahmad{Why I need the trajectory of the control inputs $\varphi_{\boldsymbol{U}}$: Because it is crucial in the definition of robustly feasible which is important for the proof of completeness. Extending the work to sample in the control inputs space, rather than the configuration space for future work.}

\begin{problem}[Optimal Motion Planning Problem (OMPP)]\label{pr:opt_mp_problem}
	Given a robot with system dynamics (\ref{eq:aff_system}), a starting configuration $\boldsymbol{q}_{start}\in\mathcal{Q}_{safe}$, a goal region $\mathcal{Q}_{goal}\subset\mathcal{Q}_{safe}$, and the obstacle-free configuration space $\mathcal{Q}_{safe}$, find $\varphi^\ast\in\mathcal{G}$ that minimizes $J(\varphi)$, i.e., $\varphi^\ast=\argmin\limits_{\varphi\in\mathcal{G},\;T\in\mathbb{R}_{>0}}J(\varphi)$.
\end{problem}
OMPP could be seen as a search in the set $\mathcal{G}$, which imposes a subsequent control problem of generating control inputs that guarantee that the produced system trajectory to be in $\mathcal{Q}_{safe}$. Moreover, OMPP is PSPACE-hard \cite{Reif1979ComplexityOT}. Kinodynamic RRT$^\ast$ \cite{Karaman2010a} is used to approximate a solution for the problem, where a tree is built incrementally starting at $\boldsymbol{q}_{start}$ and expanded towards $\mathcal{Q}_{goal}$ while satisfying differential constraints on the expanded vertices. The optimal solution $\varphi^\ast$ is approached asymptotically by rewiring the vertices of the tree.

\textbf{Our approach.} We develop an RRT$^\ast$ variant, in which we use local motion planners that generate control inputs which render the safe set $\mathcal{Q}_{safe}$ forward invariant in system (\ref{eq:aff_system}). That is, under such control inputs, for an initial state that lies in $\mathcal{Q}_{safe}$, the system trajectory will lie in $\mathcal{Q}_{safe}$ for all future times. Furthermore, we improve the sampling performance by biasing the SDF towards generating more samples in promising regions of $\mathcal{Q}_{safe}$. 

%SSSSSSSSSSSSSSSSSSSSSSSSSSSSSSSSSSSSSSSSSSSSSSSSSSSSSS:
\section{Local Motion Planning}\label{sec:local_trajP}
In this section, we develop local motion planners that we use with RRT$^\ast$ (\S \ref{sec:CBF-RRT*}). In such planners, we use controller syntheses in which CLFs and higher order CBFs (HOCBFs) are utilized to generate control inputs to steer system (\ref{eq:aff_system}) to a desired equilibrium state or to steer towards an exploratory sample while avoiding obstacles.   

\subsection{Control Lyapunov Functions}\label{subsec:CLFs}
Consider steering the state of system (\ref{eq:aff_system}) to an equilibrium state $\boldsymbol{q}_{eq}$ (i.e. $f(\boldsymbol{q}_{eq}) = 0$).

\begin{definition}[Control Lyapunov Function (CLF) \cite{Wieland2007}]\label{def:CLF}
	Let $V(\boldsymbol{q}):\mathcal{Q}\rightarrow\mathbb{R}$ be a continuously differentiable function. $V(\boldsymbol{q})$ is said to be CLF if there exist $c_1,c_2,c_3>0$, such that
	\begin{equation}\label{eq:CLF}
		\begin{aligned}
			c_1||\boldsymbol{q}-\boldsymbol{q}_{eq}||^2\leq V(\boldsymbol{q})\leq c_2||\boldsymbol{q}-\boldsymbol{q}_{eq}||^2, \\
			\dot{V}(\boldsymbol{q}) = \pounds_f V(\boldsymbol{q}) +\pounds_g V(\boldsymbol{q})\boldsymbol{u}, \\
			\inf_{\boldsymbol{u}\in\mathcal{U}} [\dot{V}(\boldsymbol{q})+c_3 V(\boldsymbol{q})]\leq0,&&\forall \boldsymbol{q}\in\mathcal{Q}.
		\end{aligned}
	\end{equation} where $\pounds_f V(\boldsymbol{q}) = \frac{\partial V(\boldsymbol{q})}{\partial \boldsymbol{q}}f(\boldsymbol{q})$, and $\pounds_g V(\boldsymbol{q}) = \frac{\partial V(\boldsymbol{q})}{\partial \boldsymbol{q}}g(\boldsymbol{q})$ denote the Lie derivatives of $V$ along $f$ and of $V$ along $g$, respectively.
\end{definition}
\begin{theorem}[\cite{Ames2019}]\label{thm:u_clf}
    Let $V(\boldsymbol{q})$ be a CLF, any Lipschitz continuous control input $\boldsymbol{u}\in\{\boldsymbol{u}\in\mathcal{U}\;|\; \pounds_f V(\boldsymbol{q}) +\pounds_g V(\boldsymbol{q})\boldsymbol{u}+c_3 V(\boldsymbol{q})\}$ asymptotically stabilizes (\ref{eq:aff_system}) to $\boldsymbol{q}_{eq}$. 
\end{theorem}

\subsection{Higher Order Control Barrier Functions}\label{subsec:HOCBFs}
Consider system (\ref{eq:aff_system}) and a differentiable function $h(\boldsymbol{q}):\,\mathcal{Q}\to\mathbb{R}$ with relative degree $\rho>0$, where $\rho$ reads as the number of times that we need to differentiate $h(\boldsymbol{q})$ until the control input $\boldsymbol{u}$ appears. Let a series of functions $\psi_j(.):\,\mathcal{Q}\to\mathbb{R}$, $j=0,\,1,\,...,\,\rho$, be defined as follows. $\psi_0\,:=\,h(\boldsymbol{q})$, and for $j\geq1$, $\psi_j\,:=\,\dot{\psi}_{j-1}\,+\,\alpha_j(\psi_{j-1})$, where $\alpha_j:\,\mathcal{Q}\to\mathbb{R}$ is a class $\mathcal{K}$ function \cite{Xiao2019c}. 

Having defined $\psi_j$, let the sets $\mathcal{C}_j$, $j=1,\dots,\rho$, be defined by $\mathcal{C}_j \,:=\,\{\boldsymbol{q}\in\mathcal{Q}\,|\,\psi_{j-1}(\boldsymbol{q})\geq0\}$.
\begin{definition}[HOCBF \cite{Xiao2019c}]\label{def:HOCBF}
	Given $\psi_0,\,...\psi_{\rho}$ with the corresponding series of sets $\mathcal{C}_1,\,...,\mathcal{C}_{\rho}$, the differentiable function $h(\boldsymbol{q})$ is said to be HOCBF of relative degree $\rho$ for system (\ref{eq:aff_system}), if there are $\alpha_1,\,...,\,\alpha_{\rho}$ class $\mathcal{K}$ functions such that
	\begin{equation}\label{eq:HOCBF_constrs}
		\begin{split}
			\pounds_f^{\rho}h(\boldsymbol{q})+\pounds_g\pounds_f^{\rho-1}h(\boldsymbol{q})\boldsymbol{u}+\frac{\partial^{\rho}h(\boldsymbol{q})}{\partial t^{\rho}} + O(h(\boldsymbol{q}))+\\\alpha_{\rho}\left(\psi_{\rho-1}(\boldsymbol{q})\right)\geq0\\
			\forall\, \boldsymbol{q}\in \mathcal{C}_1\cap\,\mathcal{C}_2\,\cap\,...\,\cap\mathcal{C}_{\rho}\quad\quad\quad\quad\quad\quad\quad\quad\quad
		\end{split}
	\end{equation}
	where $O(h(\boldsymbol{q}))$ is the partial derivatives with respect to $t$ with relative degree $\leq\rho-1$ and the remaining Lie derivatives along $f$ \cite{Xiao2019c}.
\end{definition}
\begin{theorem}[\cite{Xiao2019c}]\label{thm:HOCBF_syn}
	Let $h(\boldsymbol{q})$ be a HOCBF, then any Lipschitz continuous control input $\boldsymbol{u}$, such that, $\boldsymbol{u}\,\in\,\{\boldsymbol{u}\in\mathcal{U}\,|\,\pounds_f^{\rho}h(\boldsymbol{q})+\pounds_g\pounds_f^{\rho-1}h(\boldsymbol{q})\boldsymbol{u}+\frac{\partial^{\rho}h(\boldsymbol{q})}{\partial t^{\rho}} + O(h(\boldsymbol{q}))+\alpha_{\rho}\left(\psi_{\rho-1}(\boldsymbol{q})\right)\geq0\}$, renders the set $\mathcal{C}_1\cap\,\mathcal{C}_2\,\cap\,...\,\cap\mathcal{C}_{\rho}$ forward invariant in (\ref{eq:aff_system}).
\end{theorem}

\subsection{Formulation of Local Motion Planning}\label{subsec:frm_local_MP}

Our approach to approximate a solution for OMPP \ref{pr:opt_mp_problem} entails computing feasible motion trajectories for system (\ref{eq:aff_system}) incrementally. Producing collision-free (safe) trajectories impose a control problem of generating control inputs which guarantee that the produced trajectories are in $\mathcal{Q}_{safe}$. 

Aames \textit{et al.} \cite{Ames2014a} propose CLF-CBF-QP controller synthesis, where the control inputs are generated in a discrete-time manner. At each time step $t$, a control input is computed by solving a quadratic program (QP) subject to CLF and CBF constraints, and then applied for $\Delta t$ time to evolve system (\ref{eq:aff_system}). The CLF constraint certifies \textit{liveness} of the trajectory, that is the trajectory is progressing towards a desired equilibrium state, $\boldsymbol{q}_f$. The CBF constraints certify the \textit{safety} of the trajectory, which reads that $\mathcal{Q}_{safe}$ is forward invariant in system (\ref{eq:aff_system}). Let $V(\boldsymbol{q}(t))$ be a CLF as defined in Definition \ref{def:CLF} and $h(\boldsymbol{q}(t))$ be a HOCBF as defined in Definition \ref{def:HOCBF}, we reformulate the CLF-CBF-QP controller synthesis by constraining it with CLF and HOCBF constraints; the corresponding QP is given by:   
\begin{equation}
\begin{aligned}\label{synth:CLF-CBF-QP}
		\boldsymbol{u}_{LP}(t) = \argmin_{\boldsymbol{u}(t)\in\mathcal{U}} ||\boldsymbol{u}(t)-\boldsymbol{u}_{ref}(t)||^2+\delta^2\quad\quad\quad\quad\quad\quad\quad\\
		\textrm{s.t.}\;\begin{array}{ll}\pounds_f V(\boldsymbol{q}(t)) +\pounds_g V(\boldsymbol{q}(t))\boldsymbol{u}(t)+c_3 V(\boldsymbol{q}(t))\leq\delta\end{array} \\
		\begin{array}{ll}
			\pounds_f^{\rho}h(\boldsymbol{q}(t))+\pounds_g\pounds_f^{\rho-1}h(\boldsymbol{q}(t))\boldsymbol{u}(t)+\\\frac{\partial^{\rho}h(\boldsymbol{q}(t))}{\partial t^{\rho}} + O(h(\boldsymbol{q}(t)))+\alpha_{\rho}\left(\psi_{\rho-1}(\boldsymbol{q}(t))\right)\geq0
		\end{array}
\end{aligned}
\end{equation} where $\delta$ is a slack variable to ensure the feasibility of the HOCBF constraint if there is a conflict with the CLF constraint; $\boldsymbol{u}_{ref}(t)$ is a reference control input which could be assigned if it is desirable to track reference control inputs while certifying safety and liveness; and
$\boldsymbol{u}_{LP}(t)$ denotes the local planner control input at time $t$. 

\begin{remark}[\cite{Xiao2019c}]\label{rmrk:HOCBFs_gnrl_ofCBFs}
CBF \cite{Ames2014a} is HOCBF with $\rho=1$. In this paper, we use HOCBFs instead of CBFs to make the proposed algorithm amenable for planning for systems with relative degree $\rho\geq1$. 
\end{remark}

At each iteration of kinodynamic RRT$^\ast$, a uniform sample $\boldsymbol{q}_s\in\mathcal{Q}_{safe}$ is generated; the configuration of its nearest vertex is used as an initial condition in steering the system to a configuration in the direction of $\boldsymbol{q}_s$. If such configuration is not reachable or the trajectory to reach it is in collision with an obstacle, the sampling iteration is rejected. In \cite{Wu2020e,tedrak_reachRRT,MERL_invSafeRRT}, the reachable set of each vertex is computed explicitly and is used to guide the expansion of the tree. In this work, however, all exploration samples are accepted, and we use a variant of the CLF-CBF-QP controller synthesis (\ref{synth:CLF-CBF-QP}) that certifies producing a safe trajectory to an exploratory configuration, which might be deviated from the desired exploration configuration, see \S \ref{subbsec:Exp_LMP}. Such deviation, under some assumptions, would not violate the completeness of the algorithm (see \S \ref{sec:CBF-RRT*}) and would help with exploring the configuration space. The asymptotic optimality of RRT$^\ast$ is ensured by rewiring the vertices of the exploratory tree. For this phase, we propose to use the CLF-CBF-QP controller synthesis (\ref{synth:CLF-CBF-QP}) to generate control inputs that certify steering to exact desired configurations while certifying the safety of the system trajectory, see \S \ref{subbsec:Ext_LMP}.

\subsubsection{Exact Local Motion Planning}\label{subbsec:Ext_LMP}
In the CLF-CBF-QP (\ref{synth:CLF-CBF-QP}), we set $\boldsymbol{u}_{ref}=0$, and given $\boldsymbol{q}_{init}$ and $\boldsymbol{q}_f$, let $V(\boldsymbol{q})$ be defined as a CLF with $\boldsymbol{q}_{eq}$ being set to $\boldsymbol{q}_f$. For $\mathcal{Q}_{safe}$, assume that we are given a HOCBF $h(\boldsymbol{q})$. Using the aforementioned setting of (\ref{synth:CLF-CBF-QP}), we generate discrete control inputs to steer from $\boldsymbol{q}_{init}$ to $\boldsymbol{q}_f$. In the planner implementation, the QP (\ref{synth:CLF-CBF-QP}) is assigned to be solved with at most $\texttt{T}$ times to generate control inputs to steer (\ref{eq:aff_system}) to $\boldsymbol{q}_f$. Given that the CLF constraint is relaxed with the slack variable $\delta$, the planner might fail to steer to $\boldsymbol{q}_f$ and will stuck in local solution and in this case. In such scenario, the local motion plan will be disregarded. 

The time horizon of the produced trajectory is determined by the number of instances the QP is solved times $\Delta t$. 

\subsubsection{Exploratory Local Motion Planning}\label{subbsec:Exp_LMP}
% As detailed in \S \ref{sec:CBF-RRT*}, in the exploration phase of CBF-RRT$^\ast$ we 

In this setting we want to steer form $\boldsymbol{q}_{init}\in\mathcal{Q}_{safe}$ to an exploratory configuration $\boldsymbol{q}_f\in\mathcal{Q}$. We use a relaxed variant of the CLF-CBF-QP (\ref{synth:CLF-CBF-QP}), denoted as CBF-QP, with just HOCBF constraints. As it will become clear in shortly, the computed control inputs will generate safe trajectory to approach $\boldsymbol{q}_f$ but not necessarily steer to it exactly, which achieves an exploration task of tree-based motion planning, see \S \ref{sec:CBF-RRT*}. 

\begin{assumption}\label{asm:OL_ctrl_input}
In the absence of obstacles (i.e., $\mathcal{Q}_{safe}=\mathcal{Q}$), assume that for any $\boldsymbol{q}_f\in\mathcal{Q}$ that is reachable from any configuration $\boldsymbol{q}_{init}\in \mathcal{Q}$, the user is able to compute open-loop control inputs $\boldsymbol{u}_{OL}(t),\;t\in[0,T_{OL}]$ and time horizon $T_{OL}$ to steer system (\ref{eq:aff_system}) from $\boldsymbol{q}(0) = \boldsymbol{q}_{init}$ to $\boldsymbol{q}(T_{OL}) = \boldsymbol{q}_f$.
\end{assumption}

In the following, we detail the CBF-CLF-QP setting to implement the CBF-QP that is used with this local motion planner. Based on Assumption \ref{asm:OL_ctrl_input}, let the control inputs $\boldsymbol{u}_{OL}(t),\;t\in[0,T_{OL}]$ be computed offline in ideal setting (no obstacles in the environment) to steer to $\boldsymbol{q}_{f}$. The quadratic cost is set as $||\boldsymbol{u}(t) - \boldsymbol{u}_{OL}(t)||^2$. For $t\in[0,T_{OL}]$, $\boldsymbol{u}_{LP}(t)$ is computed by solving the aforementioned settings of QP (\ref{synth:CLF-CBF-QP}) and is applied for $\Delta t$ time duration to evolve system (\ref{eq:aff_system}).           

The utilities of using such exploratory and exact control inputs are: first, mitigate the conventional collision-checking procedure, which is computationally expensive, and second, any sample in $\mathcal{Q}$ is accepted for exploration, where the synthesis certify that the produced trajectory is in $\mathcal{Q}_{safe}$, thus, the number of samples that are used to yield an acceptable solution to Problem \ref{pr:opt_mp_problem} is reduced (see Figure \ref{fig:20RunsSimulation} the simulation experiments in \S \ref{sec:sim_exps}).

\begin{example}\label{ex:localMP_unicycle}
Consider a unicycle robot with configuration $\boldsymbol{q}=[x,\,y,\,\theta]^\top\in\mathbb{R}^2\times[-\pi,\pi]$, where $(x,y)\in\mathbb{R}^2$ and $\theta\in[-\pi,\pi]$ are the robot Cartesian position and the heading of the robot, respectively, with respect to the fixed frame $O-x_0y_0$ which is fixed at the origin. The elements of $\boldsymbol{q}$ evolve with respect to the following dynamics: $      \dot{x}=v\cos(\theta),\;\dot{y}=v\sin(\theta),\;\dot{\theta}=\omega$, where $\omega\in[\underline{\omega},\overline{\omega}],\;\underline{\omega},\overline{\omega}\in\mathbb{R}$, and $v\in[\underline{v},\overline{v}],\;\underline{v},\overline{v}\in\mathbb{R}$ are the angular velocity and the translational velocity inputs with their corresponding upper and lower bounds, respectively. We assume that the robot workspace contains obstacles that could be modeled as circles or ellipsoids.   

Given the unicycle configuration $\boldsymbol{q}$, a HOCBF of an ellipsoidal obstacle $i$ is defined as follows.
\begin{equation}\label{eq:elip_h(q)}
	\begin{aligned}
		h_i(\boldsymbol{q}(t))\, = [x(t)-x_{i},\;y(t)-y_{i}]\, E
		\begin{bmatrix}
			x(t)-x_{i}\\y(t) - y_{i},
		\end{bmatrix} -1
	\end{aligned}
\end{equation}
where $(x_i,y_i)\in\mathbb{R}^2$ is the center of the obstacle with respect to $O-x_0y_0$; and the matrix $E$ is given by
\begin{equation*}
	\begin{aligned}
		E\,=\,\left[\begin{smallmatrix}
			(\frac{\cos(\phi)}{\tilde{a}})^2\,+\,(\frac{\sin(\phi)}{\tilde{b}})^2&&-\sin(\phi)\cos(\phi)\left((\frac{1}{\tilde{b}})^2-(\frac{1}{\tilde{a}})^2\right)\\-\sin(\phi)\cos(\phi)\left((\frac{1}{\tilde{b}})^2-(\frac{1}{\tilde{a}})^2\right)&&(\frac{\sin(\phi)}{\tilde{a}})^2\,+\,(\frac{\cos(\phi)}{\tilde{b}})^2
		\end{smallmatrix}\right]
	\end{aligned}
\end{equation*}
with $\tilde{a}=a+r_r$ and $\tilde{b=b+r_r}$ being safety distances of the center of the robot along the major and minor axes, respectively; $a,b,r_r\in\mathbb{R}$ are length of the major and minor axes of the ellipsoid, and the radius of the robot, respectively, and $\phi\in[-\pi,\pi]$ is the orientation of the obstacle with respect to $O-x_0y_0$. If $a = b$, then Eq. (\ref{eq:elip_h(q)}) degenerates to a circle.

Given an initial configuration $(x_0,y_0,\theta_0)$, we want to generate motion plans for the following two cases: \textbf{(i)} steering the robot to $(x_d,y_d,\theta_d)$ using the exact local motion planner (\S \ref{subbsec:Ext_LMP}), and \textbf{(ii)} steering towards $(x_d,y_d,\theta_d)$ using the exploratory motion planner (\S \ref{subbsec:Exp_LMP}).  

\textit{Exact local motion planner formulation}. Following the approach in \cite{robotaiumGTech}, we consider controlling a look-ahead point that is $d$ distance from the center of the wheels axis and along the sagittal axis of unicycle robot. The dynamics of a look-ahead point, $(x_l,y_l)\in\mathbb{R}^2$ is given by the integrator dynamics,
\begin{equation}\label{eq:single_int}
	\begin{aligned}
		\begin{bmatrix}
		    \dot{x}_l\\\dot{y}_l
		\end{bmatrix} =
		\begin{bmatrix}
		    u_1\\u_2
		\end{bmatrix}
		=\begin{bmatrix}
			\cos\theta&-d\sin\theta\\\sin\theta&d\cos\theta
		\end{bmatrix}\begin{bmatrix}
		v\\\omega
	\end{bmatrix}.
	\end{aligned}
\end{equation} where $u_1,u_2\in\mathbb{R}$.
Let $V := ||[x_l-(x_d+d\cos{\theta}),\;y_l-(y_d+d\sin{\theta})]^\top||^2$ be a CLF with the equilibrium state set to $(x_d+d\cos{\theta}),y_d+d\sin{\theta}))$. For each obstacle we define a HOCBF (\ref{eq:elip_h(q)}) while substituting the look-ahead state variables $x_l(t)$ and $y_l(t)$ instead of $x(t)$ and $y(t)$, respectively. The CLF and HOCBF are both with relative degree $\rho=1$ with respect to the control $\boldsymbol{u} = [u_1,u_2]^\top$. We compute the HOCBF constraint using inequality (\ref{eq:HOCBF_constrs}) where $\psi_0(\boldsymbol{q})=h(\boldsymbol{q})$ and we assign $\alpha_1(\psi_0(\boldsymbol{q}))=h(\boldsymbol{q})$; in the CLF constraint in (\ref{synth:CLF-CBF-QP}) $c_3$ is set to $1$. The computed control inputs using the CLF-CBF-QP controller of the exact local motion planner could be mapped to the linear and angular velocities $v,u$ via the static map
\begin{equation}\label{eq:inv_mapping_v_omega}
	\begin{aligned}
		\begin{bmatrix}
			v\\\omega
		\end{bmatrix}=
		\begin{bmatrix}
			\cos\theta&-d\sin\theta\\\sin\theta&d\cos\theta
		\end{bmatrix}^{-1}
		\begin{bmatrix}
		    u_1\\u_2
		\end{bmatrix},
	\end{aligned}
\end{equation} where the matrix in (\ref{eq:inv_mapping_v_omega}) is always invertible unless $d=0$. In Figure \ref{fig:sim_all_paths}.\textbf{b} we show the generated trajectory of the look-ahead state using the exact local motion planner control inputs.

\textit{Exploratory local motion planner formulation}. Similar to the exact steering formulation, we consider controlling the look-ahead point. In the exploratory CBF-QP (see \S \ref{subbsec:Exp_LMP}) we compute $\boldsymbol{u}_{OL}$ as follows. $(x_{l,0},y_{l,0})=(x_0+d\cos{\theta_0},y_0+d\sin{\theta_0})$ and $(x_{l,d},y_{l,d})=(x_d+d\cos{\theta_d},y_d+d\sin{\theta_d})$ are the initial and desired configurations of the look ahead point, respectively, given the integrator dynamics (\ref{eq:single_int}) we define $\boldsymbol{u}_{OL}$ as piecewise linear controls that represents the line between $(x_{l,0},y_{l,0})$ and $(x_{l,d},y_{l,d})$. In Figure \ref{fig:sim_all_paths}.\textbf{c} we show the produced trajectory of the look-ahead state using the exploratory local motion planner control input, where the trajectory is deviated from following $\boldsymbol{u}_{OL}$ due to the presence of obstacles.      
\end{example}

%SSSSSSSSSSSSSSSSSSSSSSSSSSSSSSSSSSSSSSSSSSSSSSSSSSSSSSSS
\section{CBF-RRT$^\ast$}\label{sec:CBF-RRT*}
In this section we detail the algorithmic formulation of the proposed algorithm, CBF-RRT$^\ast$, which approximates a solution of the OMPP \ref{pr:opt_mp_problem}. The exploratory and exact local motion planners (see \S \ref{subbsec:Exp_LMP}, and \S \ref{subbsec:Ext_LMP}) are used to expand the RRT tree and to rewire the tree, respectively. We show that, under some assumptions, the probabilistic completeness of RRT$^\ast$ is preserved using such local motion planning. 

\subsection{The Algorithm}\label{subsec:CBF-RRT*_alg} 
Considering tree $\mathcal{T}=(\mathcal{V},\mathcal{E})$ on $\mathcal{Q}_{safe}$, with vertices set $\mathcal{V}\subset\mathcal{Q}_{safe}$ and edges $\mathcal{E}=\mathcal{V}\times\mathcal{V}$, we define the following primitive functions that are used in the proposed work: \textbf{(i)} $\texttt{Sample}(\mathcal{G},\texttt{adapFlag}):\mathcal{G}\times\{\texttt{True},\,\texttt{False}\}\to\mathcal{Q},$, given a number of $\varphi\in\mathcal{G}$ and a Boolean variable $\texttt{adapFlag}$, the function returns a sample in $\mathcal{Q}$. If $\texttt{adapFlag}=\texttt{False}$, the function returns a uniform sample from $\mathcal{Q}$, otherwise the SDF will be adapted (see \S \ref{sec:AdapCBF-RRT*}) and will be used to generate a sample in $\mathcal{Q}$. \textbf{(ii)} $\texttt{Comp\_uOL}(\boldsymbol{q}_{s},v):\mathcal{Q}_{safe}\times\mathcal{V}\to\mathcal{U}$, given sample $\boldsymbol{q}_s$, vertex $v$ and considering Assumption \ref{asm:OL_ctrl_input}, the function computes the open-loop  control inputs $\boldsymbol{u}_{OL}(t)\in\mathcal{U},\;t\in[0,T_{OL}]$ and a time horizon $T_{OL}$ to steer from vertex $v$ towards $\boldsymbol{q}_s$. \textbf{(iii)} $\texttt{ExpLPlanning}(v,\boldsymbol{u}_{OL},T_{OL}):\mathcal{V}\times\mathcal{U}\times\mathbb{R}_{>0}\to\mathcal{V}$, given vertex $v$ and control inputs $\boldsymbol{u}_{OL}$, the function steers system (\ref{eq:aff_system}) form vertex $v$ using the exploratory CBF-QP controller synthesis (see \S \ref{subbsec:Exp_LMP}) with $\boldsymbol{u}_{ref}=\boldsymbol{u}_{OL}$, and then establishes a vertex, $v_{new}$, at the last configuration of the produced trajectory, which, as detailed in \S \ref{sec:local_trajP}, is certified to be in $\mathcal{Q}_{safe}$. \textbf{(iv)} $\texttt{ExtLPlanning}(v_1,v_2):\mathcal{V}\times\mathcal{V}\to\mathcal{V}$, given the vertices $v_1$ and $v_2$, the function steers from $v_1$ to $v_2$ using the exact CLF-CBF-QP synthesis (see \S \ref{subbsec:Ext_LMP}). 

CBF-RRT$^\ast$ is initialized with a root vertex, $v_{start}$, at $\boldsymbol{q}_{start}$ (Line \ref{line:alg1_init} in Algorithm \ref{alg:Ada_CBF-RRT*}). Exploration is done by sampling $\boldsymbol{q}_s\in\mathcal
{Q}$, which is used guide the expansion of its nearest vertex, $v_{nearest}$ (Line \ref{line:1st_v_nearest} - Line \ref{line: alg1_ SafeSteer}). First, using $\texttt{Comp\_uOL}$ we compute $\boldsymbol{u}_{OL}(t),t\in[0,T_{OL}]$ that, if $||\boldsymbol{q}_s-v_{nearest}<\eta||$, steer system (\ref{eq:aff_system}) from $v_{nearest}$ to $\boldsymbol{q}_{s}$, otherwise, steer system (\ref{eq:aff_system}) from $v_{nearest}$ to $\boldsymbol{q}_{new}$ such that $||\boldsymbol{q}_{new}-v_{nearest}=\eta||$ and in the direction of $\boldsymbol{q}_s$, where $\eta\in\mathbb{R}_{R>0}$(Line \ref{line:alg1_compuOL}). As it will become clear in the completeness details (\S \ref{subsec:CBFRRT*_completenss}), we assign $\eta=\frac{\varepsilon}{4}+\mu+2\iota$, where $\varepsilon$ is a parameter imposed by the robot environment, $\mu$ is a parameter measured by tuning the HOCBF, and $0<\iota<\frac{\varepsilon}{4}-\mu$. Second, the computed $\boldsymbol{u}_{OL}$ is used with the exploratory local motion planner $\texttt{ExpLPlanning}$ (see \S \ref{subbsec:Exp_LMP}) to extend to $v_{new}$, where $v_{new}$ and the produced trajectory to reach it are certified by construction to be in $\mathcal{Q}_{safe}$ (Line \ref{line: alg1_ SafeSteer}). 

The ideal case of the exploration phase is to steer to a new configuration ($\boldsymbol{q}_{new}$) such that $||v_{nearest}-\boldsymbol{q}_{new}||=\eta$ and in the direction of sample $\boldsymbol{q}_s$, however, if such configuration lies within or close to an obstacle, the produced trajectory will deviate from reaching the desired configuration. Such deviation, however, is acceptable under some assumptions to preserve the completeness of the algorithm, see Theorem \ref{Thm:cmpltness_CBF-RRT*}. Moreover, since the trajectories are guaranteed to be safe, no explicit collision check is needed, which reduces the computational burden of the algorithm.  

The rewiring procedure (Line \ref{line:alg1_ begining_rewiring} - Line \ref{line:alg1_ end_rewiring}) is similar to the conventional rewiring of RRT$^\ast$ (see ~\cite{KaramanRRTstarIJRR} for details). Rewiring vertex $v_1\in\mathcal{V}$ to a vertex that is reachable from, $v_2\in\mathcal{V}$, is accomplished through the exact CLF-CBF-QP control synthesis (\S \ref{subbsec:Ext_LMP}), see Line \ref{line:alg1_ext_local_trajP1} and Line \ref{line:alg1_ext_local_trajP2}. 

\begin{remark}\label{rmrk:Raduis_of_AsymOpt_of_RRTstar}
Theorem $1$ in \cite{rev_asymptoticRRTstar} concludes that the asymptotic optimality of a solution of OMPP, that is generated using Kinodynamic RRT$^\ast$, is guaranteed by the following condition: for vertex $v\in\mathcal{V}$, the vertices that lie within a $d$-dimensional hypersphere of radius $\lambda(\log(|\mathcal{V}|)/|\mathcal{V}|)^{1/(d+1)}$ are considered for searching for better parent vertex for $v$, where $\lambda\in\mathbb{R}_{>0}$ and $|\mathcal{V}|$ is the number of vertices of tree $\mathcal{T}$. We assign such radius in the rewiring procedure for CBF-RRT$^\ast$ (Line \ref{line:alg1_ball_radi} in Algorithm \ref{alg:Ada_CBF-RRT*}). Having used local motion planners based on the CBF-QP and CLF-CBF-QP controller syntheses (see \S \ref{subbsec:Exp_LMP} and \S \ref{subbsec:Ext_LMP}), however, requires further investigation to ensure that the asymptotic optimality will indeed be preserved, which we leave as future work.  
\end{remark}

%%%%  Algorithms' pseudo:
\setlength{\textfloatsep}{15pt}
\begin{algorithm}[t]\scriptsize
	\bf{Input:} $\boldsymbol{q}_{start}$; $\mathcal{Q}_{goal}$; $\mathcal{Q}_{safe}$; $N$, $e_l$, and $\Delta t$\\
	\textbf{Initialization:}  $v_{start} =(\boldsymbol{q}_{init},\texttt{index}=0)$, $i=1$, $\mathcal{V}=\{v_{init}\}$, $\mathcal{E}=\emptyset$, $\mathcal{G} = \emptyset$, $\texttt{GoalReached} = \texttt{False}$, $\,\texttt{adapFlag}=\texttt{True}$, $\,\texttt{optDensityFlag}\,=\,\texttt{False}$, and $r = \eta$\label{line:alg1_init}\\
	$\mathcal{T}\leftarrow(\mathcal{V},\,\mathcal{E})$\\
	\While{$i<N$}
	{$\boldsymbol{q}_{s}\leftarrow\texttt{Sample}(\mathcal{G},\texttt{adapFlag})$ \label{Line: Alg1.Sampling}\\
		$v_{nearest} \leftarrow \texttt{Nearest}(\boldsymbol{q}_s)$\label{line:1st_v_nearest}\\

		$\boldsymbol{u}_{OL} = \texttt{Comp\_uOL}(\boldsymbol{q}_{s},v_{nearest})$\label{line:alg1_compuOL}\\
		$\mathcal{V}\leftarrow\mathcal{V}\cup\{v_{new}\leftarrow\texttt{ExpLPlanning}(v_{nearest},\boldsymbol{u}_{OL})\}$ \label{line: alg1_ SafeSteer}\\
		$r = \min\{\lambda(\log(|\mathcal{V}|)/|\mathcal{V}|)^{1/(d+1)},\eta\} $\label{line:alg1_ball_radi}\\

		$\mathcal{V}_{near}\leftarrow\texttt{Near}(\mathcal{T},r,v_{new})$\\
		%%% Choose Parent:
		\ForEach{$v_{near}\in\mathcal{V}_{near}$}
		{	\label{line:alg1_ begining_rewiring}
			$v'\leftarrow\texttt{ExtLPlanning}(v_{near},v_{new})$\label{line:alg1_ext_local_trajP1}\\
			$c'=v_{near}.\texttt{cToCome}+\texttt{Cost}(v',v_{near})$\\
			\If{$c'<c_{min}$}
			{
				$v_{new}'\leftarrow v'$;
				$v_{min}\leftarrow v_{near}$;
				$c_{min}\leftarrow c'$
			}
		}

		$\mathcal{T}\leftarrow\texttt{AddChild}(\mathcal{T},v_{min},v_{new})$\\

		%%%%%%% Rewiring procedure:
		\ForEach{$v_{near}\in \mathcal{V}_{near}$}
		{
% 			$s_{near} = \texttt{CompFreeState}(v_{near},v_{new})$\\
			$v'\leftarrow \texttt{ExtLPlanning}(v_{new},v_{near})$\label{line:alg1_ext_local_trajP2}\\
			\If{($v_{new}.\texttt{cToCome}+Cost(v_{new},v')< v_{near}.\texttt{cToCome})$}
			{
				$\mathcal{T}\leftarrow\texttt{Reconnect}(v_{new},v_{near},\mathcal{T})$\\
				$\texttt{UpcToCome}(v_{near},\texttt{cToCome}(v_{new}+Cost(v')))$\label{line:alg1_ end_rewiring}
			}
		}
		$\mathcal{T},\;\mathcal{G}\leftarrow\texttt{extToGoal}(\mathcal{T},v_{new},\texttt{adapFlag})$;	$i\leftarrow i+1$ \label{line:alg1_ CBF-RRT*_ext2goal}}\Return$\mathcal{T}$
	\caption{Adaptive CBF-RRT$^\ast$}
	\label{alg:Ada_CBF-RRT*}
\end{algorithm}
\setlength{\textfloatsep}{14pt}
\setlength{\textfloatsep}{15pt}
\begin{algorithm}[t]\scriptsize
	$\texttt{u}\sim \texttt{Uniform}(0,1)$\\
	\eIf{$\texttt{u}\;\leq 0.5\; \wedge\;\mathcal{G}\neq\emptyset$}
		{
			\eIf{\texttt{optDensityFlag}}
				{
					$X\sim\hat{\mathfrak{g}}^\ast(\boldsymbol{q})$\\
					\Return $(\boldsymbol{q})$\label{line:alg_ smple_from_opt_est}
				}
				{
					\eIf{$\texttt{mod}(|\mathcal{G}|, n_v)=0$}
						{
					    $\mathfrak{E}\leftarrow\texttt{Quantile}(\mathcal{G},\varrho)$\Comment{Assign the elite set}\\\label{line:alg2_strt_discrtzng}
						$\hat{\mathfrak{g}}(\boldsymbol{q})\leftarrow\texttt{CE\_Estimation}(\mathfrak{E},m)$\Comment{Compute PDF of $\mathfrak{E}$}\label{line:Alg2_CE_estimation}
							\\ \Return$\boldsymbol{q}\sim\hat{\mathfrak{g}}(\boldsymbol{q})$ \label{line:alg2_smpl_form_g_est}
						}%Else:
						{
							\Return $\boldsymbol{q}\sim\texttt{Uniform}(\mathcal{Q})$
						}
				}
		}{\Return $\boldsymbol{q}\sim\texttt{Uniform}(\mathcal{Q})$}
	\caption{$\boldsymbol{q}_{s}\leftarrow\texttt{Sample}(\mathcal{G},\mathcal{T},m)$}
	\label{alg:Sample}
\end{algorithm}
\setlength{\textfloatsep}{14pt}

\subsection{Probabilistic Completeness of the Algorithm}\label{subsec:CBFRRT*_completenss}
In this section we establish, under some assumptions, the probabilistic completeness of CBF-RRT$^\ast$.

We formulate some definitions that are needed to establish the main completeness result (Theorem \ref{Thm:cmpltness_CBF-RRT*}). For any $\varphi\in\mathcal{G}$, we define $\varphi_{\boldsymbol{q}}:=\{\mathbf{proj}_{\mathcal{Q}}(\varphi(t,T))|\varphi\in\mathcal{G},\;t\in[0,T]\}$ and $\varphi_{\boldsymbol{u}}:=\{\mathbf{proj}_{\mathcal{U}}(\varphi(t,T))|\varphi\in\mathcal{G},\;t\in[0,T]\}$, where $\mathbf{proj}_{\mathcal{U}}:\mathcal{U}\times\mathcal{Q}\to\mathcal{U}$ and $\mathbf{proj}_{\mathcal{Q}}:\mathcal{U}\times\mathcal{Q}\to\mathcal{Q}$ are the projection of the control inputs and the produced trajectory of $\varphi$, respectively, i.e.,  $\mathbf{proj}_{\mathcal{U}}(\varphi(t,T)) = \boldsymbol{u}(t)$ and $\mathbf{proj}_{\mathcal{Q}}(\varphi(t,T)) = \boldsymbol{q}(t)$. 

Following \cite{KaramanRRTstarIJRR}, we say that OMPP \ref{pr:opt_mp_problem} is robustly feasible with minimum clearance $\varepsilon>0$, if there exist control inputs $\varphi_{\boldsymbol{u}}$ which produce trajectory $\varphi_{\boldsymbol{q}}$, and $\varphi\in\mathcal{G}$, such that the distance between any configuration $\boldsymbol{q}\in\varphi_{\boldsymbol{q}}$ and any obstacle configuration $\boldsymbol{q}_{o}\in\mathcal{Q}\setminus\mathcal{Q}_{safe}$ is at least $\varepsilon/2$.  

\begin{theorem}[\cite{kinodynamic_wSimulationFrwrd}]\label{Thm:Bndd_trajs}
Let $\boldsymbol{q}(t)\in\mathcal{Q}_{safe},t\in[0,T]$ and $\boldsymbol{q}^\prime(t)\in\mathcal{Q}_{safe},t\in[0,T]$ be trajectories of system (\ref{eq:aff_system}) under $\boldsymbol{u}(t),\;t\in[0,T]$ and $\boldsymbol{u}^\prime(t),\;t\in[0,T]$ control inputs, respectively, such that they have the same initial configuration $\boldsymbol{q}(0) = \boldsymbol{q}^\prime(0)$, then, for $a,b\in\mathbb{R}_{>0}$, the following bound holds 
\begin{equation}\label{ineq:BoundTraj}
||\boldsymbol{q}(T) - \boldsymbol{q}^\prime(T)||<a\;e^b\;\sup_t(||\boldsymbol{u}(t)-\boldsymbol{u}^\prime(t)||)    
\end{equation}
\end{theorem}

\begin{remark}\label{rmrk:TunableParamsOfCBFs}
Consider steering from any $\boldsymbol{q}_s\in\mathcal{Q}_{safe}$ towards any reachable $\boldsymbol{q}_f\in\mathcal{Q}_{safe}$ using $\boldsymbol{u}_{LP}(t),\;t\in[0,T_{OL}]$, that is computed by the CBF-QP synthesis (see \S \ref{subbsec:Exp_LMP}). One can tune the constants of the class $\mathcal{K}$ functions, $\alpha_1,\,...,\,\alpha_{\rho}$, of the HOCBF (see Defintion \ref{def:HOCBF}) such that the produced $\boldsymbol{u}_{LP}$ let the bound (\ref{ineq:BoundTraj}) be given as $||\boldsymbol{q}_f - \boldsymbol{q}_f^\prime||=\mu<\frac{\varepsilon}{4}$, where $\boldsymbol{q}_f$ and $\boldsymbol{q}^\prime_f$  are the configurations at time $T_{OL}$ of the produced trajectories under the open-loop control inputs $\boldsymbol{u}_{OL}$ and the CBF-QP control inputs $\boldsymbol{u}_{LP}$, respectively, and $\mu\in\mathbb{R}_{>0}$.    
\end{remark}
\begin{theorem}
\label{Thm:cmpltness_CBF-RRT*}
 CBF-RRT$^\ast$ is probabilistically complete.
\end{theorem}
% \Tron{Needs the following conditions: tune CBF to make $\delta<\varepsilon/4$, make the sampling period $T$ small enough (and hence the number of samples $m$ large enough)/tune CBF so that $\norm{q_f-q'_f}<\delta$}
\begin{proof}
    RRT$^\ast$ completeness is implied by the completeness of RRT (see Theorem 23 in \cite{KaramanRRTstarIJRR}). Following this result, we prove the completeness of CBF-RRT, then the completeness of CBF-RRT$^\ast$ will follow directly since, using the same sequence of samples, its tree is the rewired CBF-RRT tree and the fact that the rewiring procedure is accomplished by exact steering. 
    
    CBF-RRT is implemented by mitigating the rewiring procedure, Lines \ref{line:alg1_ begining_rewiring}-\ref{line:alg1_ end_rewiring} in Algorithm \ref{alg:Ada_CBF-RRT*}. Given that $\texttt{ExpLPlanning}$ (the only local motion planner in CBF-RRT) generates control inputs using the CBF-QP controller synthesis (\S \ref{subbsec:Exp_LMP}), the extended trajectories of the tree are guaranteed to be collision-free, hence, by leveraging Theorem 2 in \cite{RRTcompleteness}, we only need to prove that the incremental trajectory will propagate to reach $\mathcal{Q}_{goal}$.

    Assume that the trajectory $\varphi_{\boldsymbol{q}}$ of the solution of OMPP with $\varepsilon$ clearance has a length $L$. Considering $m+1$ equidistant configurations $\boldsymbol{q}_{i}\in\varphi_{\boldsymbol{q}},\;i=1,\dots,m+1$, where $m=\left\lfloor\frac{4L}{\varepsilon} \right\rfloor$, we define a sequence of balls with radius $\varepsilon/4$ that are centered at these configurations. For configuration $\boldsymbol{q}_i$, such ball is given by: $\mathfrak{B}_{\frac{\varepsilon}{4}}(\boldsymbol{q}_i):=\{\boldsymbol{q}_b\;|\;||\boldsymbol{q}_i-\boldsymbol{q}_b||\leq\frac{\varepsilon}{4}\}$, see Figure \ref{fig:conseqB's} for illustration of two consecutive balls. For the consecutive configurations $\boldsymbol{q}_{i},\boldsymbol{q}_{i+1}\in\varphi_{\boldsymbol{q}}$, we want to prove that starting from $\boldsymbol{q}_s\in\mathfrak{B}_{\frac{\varepsilon}{2}}(\boldsymbol{q}_i)$ the exploratory local motion planner $\texttt{ExpLPlanning}$ is able to generate a motion trajectory that its end configuration $\boldsymbol{q}_f^\prime$ fall in $\mathfrak{B}_{\frac{\varepsilon}{4}}(\boldsymbol{q}_{i+1})$. Given Remark \ref{rmrk:TunableParamsOfCBFs}, we assign $\eta = \frac{\varepsilon}{4}+\mu+2\iota$ and $0<\iota<\frac{\varepsilon}{4}-\mu$. Accordingly, we assign $\mathfrak{B}_{\eta}(\boldsymbol{q}_s)$ and $\mathfrak{B}_{\frac{\varepsilon}{4}-\mu-\iota}(\boldsymbol{q}_{i+1})$ at $\boldsymbol{q}_s$ and $\boldsymbol{q}_{i+1}$, respectively. Let $\mathcal{S}:=\mathfrak{B}_{\eta}(\boldsymbol{q}_s)\cap\mathfrak{B}_{\frac{\varepsilon}{4}-\mu-\iota}(\boldsymbol{q}_{i+1})$ denotes the successful potential end-configurations set, which is depicted as the magenta region in Figure \ref{fig:conseqB's}. For any $\boldsymbol{q}_f\in\mathcal{S}$, $\texttt{ExpLPlanning}$ will succeed to generate trajectories that fall in $\mathfrak{B}_{\mu}(\boldsymbol{q}_f)\subset\mathfrak{B}_{\frac{\varepsilon}{4}}(\boldsymbol{q}_{i+1})$, which is depicted as the union of the green and magenta regions in Figure \ref{fig:conseqB's}. Let $|.|$ denotes the Lebesgue measure, then, for $\boldsymbol{q}_s$, the probability of generating configurations in $\mathcal{S}$ is $p=\frac{|\mathcal{S}|}{|\mathcal{Q}|}$ and is strictly positive. The probability $p$ represents the success probability of the $k$ Bernoulli trials process that models generating $m$ successful outcomes of sampling configurations that incrementally reach $\mathcal{Q}_{goal}$ \cite{RRTcompleteness}. The rest of the proof follows directly the proof of Theorem 1 in \cite{RRTcompleteness}.
\end{proof}

\begin{figure}\centering
	\includegraphics[width=0.7\linewidth]{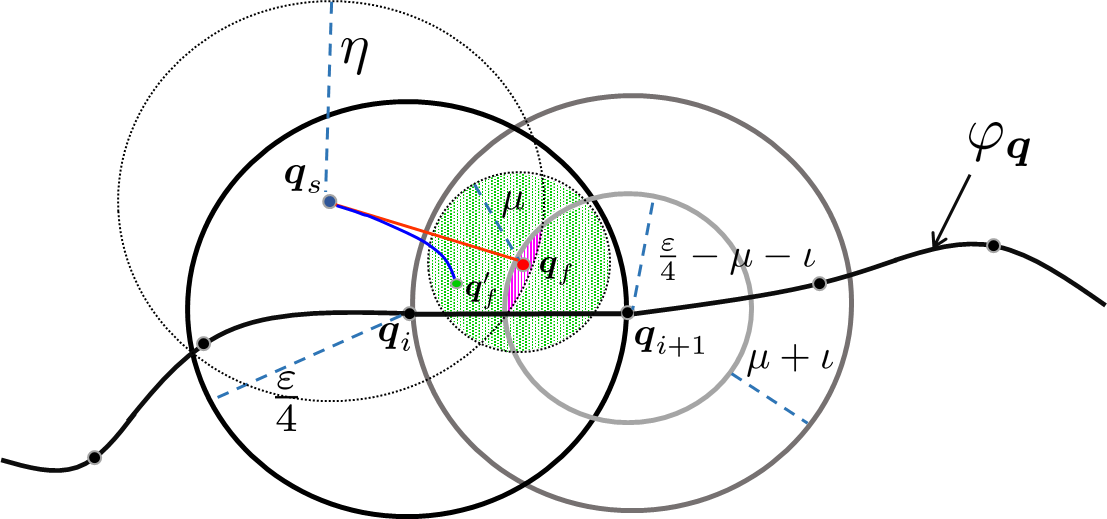}
	\caption{Depiction of two consecutive balls to illustrate Theorem \ref{Thm:cmpltness_CBF-RRT*}}
	\label{fig:conseqB's}
\end{figure}

\section{Adaptive Sampling for CBF-RRT$^\ast$}\label{sec:AdapCBF-RRT*}
We leverage CBF-RRT$^\ast$ with an adaptive sampling procedure, in which we use CEM to focus sampling around the optimal trajectory $\varphi_{\boldsymbol{q}}^\ast$. The motivation behind this approach is to approximate the solution of the OMPP \ref{pr:opt_mp_problem} with a fewer number of samples by focusing the sampling in promising regions of $\boldsymbol{Q}_{safe}$. 

\subsection{Adaptive Sampling using the cross-entropy Method}\label{subsec:CEM}
% We propose to use CEM to improve the SDF (SDF) of CBF-RRT$^\ast$.

CEM \cite{Rubinstein1999} has been used to estimate the probability of rare events using IS. Conventional simulation methods, e.g. Monte-Carlo simulation, are prone to incorrectly estimate such probabilities to be zero \cite{Rubinstein1999}. CEM is a multi-stage stochastic optimization algorithm that iterates upon two steps: first, it generates samples from a current (parametric) SDF and computes the cost of each sample; second, it chooses an \emph{elite subset} of the generated samples for which their cost is below some threshold; finally, the elite subset of samples is used to estimate a probability density function (PDF) as if the elite samples were drown as i.i.d samples from such PDF. The estimated PDF will be used as the SDF for the next iteration. The algorithm terminate when it converges to a limiting PDF. It has been proven in \cite{Hu2007} that CEM with parametric SDF converges to a limiting distribution.

Going into more technical details, let a random variable $Q:\Omega\rightarrow\mathcal{Q}$ be defined over the probability space $(\Omega,\mathcal{F},P)$, where $\Omega$ is the sample space, $\mathcal{Q}$ is the range space, $\mathcal{F}$ denotes the $\sigma$-algebra subset of $\mathcal{Q}$, and $P$ is the probability measure over $\mathcal{F}$. CEM aims to find rare events with probability $P(\mathcal{J}(\boldsymbol{q})\leq\gamma)$, where $\gamma\in\mathbb{R}_{>0}$ is an optimal cost threshold and $\mathcal{J}:\mathcal{Q}\to\mathbb{R}_{>0}$ is the cost of a sampled solution $\boldsymbol{q}$. Computing $P(\mathcal{J}(\boldsymbol{q})\leq\gamma)$ is equivalent to computing the expectation $E[I(\{\mathcal{J}(\boldsymbol{q})<\gamma\})]$, where $I(.)$ is the indicator function. 

The work in \cite{Rubinstein1999} proposes to evaluate the expectation $E[I(\{\mathcal{J}(\boldsymbol{q})<\gamma\})]$ using the following estimator: $\hat{\ell}=\frac{1}{N}\sum\limits_{i=1}^{N}I(\{\mathcal{J}(\boldsymbol{q})<\gamma\})\frac{\mathfrak{f}(\boldsymbol{q}_i)}{\mathfrak{g}(\boldsymbol{q}_i)}$, where $\mathfrak{f}(\boldsymbol{q}_i)$ is the PDF of a sampled solution, $\boldsymbol{q}_i$, and $\mathfrak{g}(\boldsymbol{q}_i)$ is an underlying IS PDF. 
Choosing $\mathfrak{g}(\boldsymbol{q})^\ast=I(\{\mathcal{J}(\boldsymbol{q})<\gamma\})\mathfrak{f}(\boldsymbol{q})/\hat{\ell}$ will yield the best estimate of $\hat{\ell}$. However, this solution is hypothetical, since it involves $\hat{\ell}$, which is the entity that we want to estimate in the first place. Instead, $\mathfrak{g}(\boldsymbol{q})^\ast$ is solved in a multi-stage manner, where at each stage the elite subset of samples is used to estimate $\mathfrak{g}(\boldsymbol{q})$ until the CE between $\mathfrak{g}(\boldsymbol{q})^\ast$ and $\mathfrak{g}(\boldsymbol{q})$ is minimized. The CE is related to the Kullback-Leibler divergence, $D_{KL}(\mathfrak{g}^\ast(\boldsymbol{q})||\mathfrak{g}(\boldsymbol{q})) = \int_{\mathcal{Q}} \mathfrak{g}^\ast(\boldsymbol{q}) \ln (\mathfrak{g}^\ast(\boldsymbol{q})/\mathfrak{g}(\boldsymbol{q}))\; d\boldsymbol{q}$, and minimizing it implies minimizing the CE.

\subsection{Adaptive CBF-RRT$^\ast$}
 In the context of CBF-RRT$^\ast$, one could ask the following question: what is the probability of sampling configurations that lie precisely on $\varphi_{\boldsymbol{q}}^\ast$ (the solution of OMPP \ref{pr:opt_mp_problem})? It can be easily seen that it is an extremely small probability. Kobalirov \cite{Kobilarov2012e} proposes to use CEM to estimate the probability of generating samples that lie on $\varphi_{\boldsymbol{q}}^\ast$ using a mixture of Gaussian models for the proposal distribution $\mathfrak{g}(\boldsymbol{q})$. For the planning problem, however, it is hard to know, prior to planning, how promising regions of $\mathcal{Q}$ are distributed in order to choose a suitable number of Gaussian models.

The challenge above has motivated us to use a nonparametric density estimate, namely the weighted \textit{Gaussian Kernel Density Estimate} (WGKDE) \cite{Botev2011}, instead of a mixture of Gaussian models.

To improve the SDF of CBF-RRT$^\ast$ using CEM, we need to generate a population of approximated solutions of the OMPP \ref{pr:opt_mp_problem}. To accumulate such population of solutions, $\texttt{extToGoal}$ procedure (Line \ref{line:alg1_ CBF-RRT*_ext2goal} in Algorithm \ref{alg:Ada_CBF-RRT*}) attempts to steer system (\ref{eq:aff_system}) from vertex $v_{new}$ to $\boldsymbol{q}_g\in\mathcal{Q}_{goal}$ using $\texttt{ExpLPlanning}$. If the final configuration of the produced trajectory lies in $\mathcal{Q}_{goal}$, a vertex, $v_g$, at that configuration, is created and added to $\mathcal{V}$ and an edge, $(v_{new},v_g)$, is added to $\mathcal{E}$. Accordingly, the generated control inputs and system trajectory, $\varphi$, is added to $\mathcal{G}$.

%  At an iteration, when the CBF-RRT$^\ast$ tree is grown with sufficient number of vertices in $\mathcal{Q}_{goal}$ (i.e. $\mathcal{G}\neq\emptyset$), 
 
 Adapting the SDF of CBF-RRT$^\ast$ is detailed in Algorithm \ref{alg:Sample}. Consider an iteration of CBF-RRT$^\ast$ (Algorithm \ref{alg:Ada_CBF-RRT*}) with accumulated trajectories to $\mathcal{Q}_{goal}$ (i.e. $\mathcal{G}\neq\emptyset$), 
 the elite set, $\mathfrak{E}$, is assigned by choosing the trajectories of all $\varphi\in\mathcal{G}$ with $J(\varphi)\leq\gamma$, i.e $\mathfrak{E}=\{\varphi_{\boldsymbol{q}}\,|\,(\varphi_{\boldsymbol{u}},\varphi_{\boldsymbol{q}})=\varphi\in\mathcal{G};\;J(\varphi)\leq\gamma\}$. We pick $\gamma$ as the $\varrho^{th}$ percentile cost of $\varphi\in\mathcal{G}$; Rubinstein \textit{et al.} \cite{Rubinstein1999} suggests to assign $\varrho\in[0.01,0.1]$. Since the SDF of CBF-RRT$^\ast$ samples in $\mathcal{Q}$, we will use a sparse set of configurations of the elite trajectories ($\mathfrak{E}$) to estimate a PDF that will be used as an SDF for the next iteration. Let $\texttt{d\_elite}$ be a set of pairs of $e$ discrete configurations of each $\varphi_{\boldsymbol{q}}\in\mathfrak{E}$ with assigned cost of each configuration as the cost of the corresponding elite trajectory, i.e., $\texttt{d\_elite}(\mathfrak{E},e)=\{(\boldsymbol{q}_i,\mathcal{J}(\boldsymbol{q}_i))|i\in\{1,\dots,m\},\boldsymbol{q}_i\in\varphi_{\boldsymbol{q}},\;\forall\varphi_{\boldsymbol{q}}\in\mathfrak{E},\mathcal{J}(\boldsymbol{q}_i)=J(\varphi)\}$. The WGKDE of the discretized elite trajectories is computed by: $\hat{\mathfrak{g}}(\boldsymbol{q})\; = \; \sum\limits_{(\boldsymbol{q}_i,\mathcal{J}(\boldsymbol{q}_i))\in\texttt{d\_elite}(\mathfrak{E},m)}\tilde{w}_i\;K(\boldsymbol{q})$, where the normalized weight $\tilde{w}_i$ and the Gaussian kernel function $K(\boldsymbol{q})$ are computed, respectively, by: $\tilde{w}_i = 1-\frac{\mathcal{J}(\boldsymbol{q}_i)}{\sum\limits_{(\boldsymbol{q}_j,\mathcal{J}(\boldsymbol{q}_j))\in\texttt{d\_elite}(\mathfrak{E},m)}(\mathcal{J}(\boldsymbol{q}_j))}$ and $K_i(\boldsymbol{q})  = \frac{1}{\sqrt{2\,\pi}\sigma}\,exp{\left(\frac{-||\boldsymbol{q}-\boldsymbol{q}_i||^2}{2\sigma^2}\right)}$. The procedure $\texttt{CE\_Estimation}(\mathfrak{E},m)$ (Line \ref{line:Alg2_CE_estimation} in Algorithm \ref{alg:Sample}) performs the WGKDE from the elite trajectories and checks if the KL-divergence between the current estimate and the previous estimate bellow a certain threshold and update $\texttt{optDensityFlag}$ accordingly. 
 
As more vertices are added to $\mathcal{T}$, more trajectories that reach the goal are used in adapting the SDF. Finally, the algorithm converges to a limiting PDF (where in this case $\texttt{optDensityFlag}$ is set to $\texttt{True}$), which will be used as the final SDF of CBF-RRT$^\ast$.

%SSSSSSSSSSSSSSSSSSSSSSSSSSSSSSSSSSSSSSSSSSSSSSSSSSSSSdistribtionSSS
\section{Simulation Example}\label{sec:sim_exps}
We consider generating motion plans using CBF-RRT, RRT$^\ast$, CBF-RRT$^\ast$, Adaptive CBF-RRT$^\ast$, and the CLF-CBF-QP-based exact motion planner for the unicycle drive robot of Example \ref{ex:localMP_unicycle}. The generated paths are depicted in Figure \ref{fig:sim_all_paths}, where the Adaptive CBF-RRT$^\ast$ (shown in solid blue path) appears to be the smoothest path because the algorithm keeps the extensions of the vertices to the goal as part of the tree. Even though keeping such extensions requires additional memory, they help to produce acceptable paths with fewer vertices, see Figure \ref{fig:20RunsSimulation}; moreover, these extensions are exploited for efficient sampling.

Figure \ref{fig:20RunsSimulation} shows the evolution of the path length with respect to the number of tree vertices of each implementation. For 20 independent runs of Adaptive CBF-RRT$^\ast$ the algorithm needed, on average, $392$ vertices to converge to a limiting sampling distribution, which leads to more efficient refinement of the path, see Figure \ref{fig:Demo_of_adaptiv_CBF-RRT*} for an illustration of the evolution of the IS density function. On the other hand, the other algorithms were able to find a path after the $200^{th}$ vertex.

\begin{figure}\centering
	\includegraphics[width=0.8\linewidth]{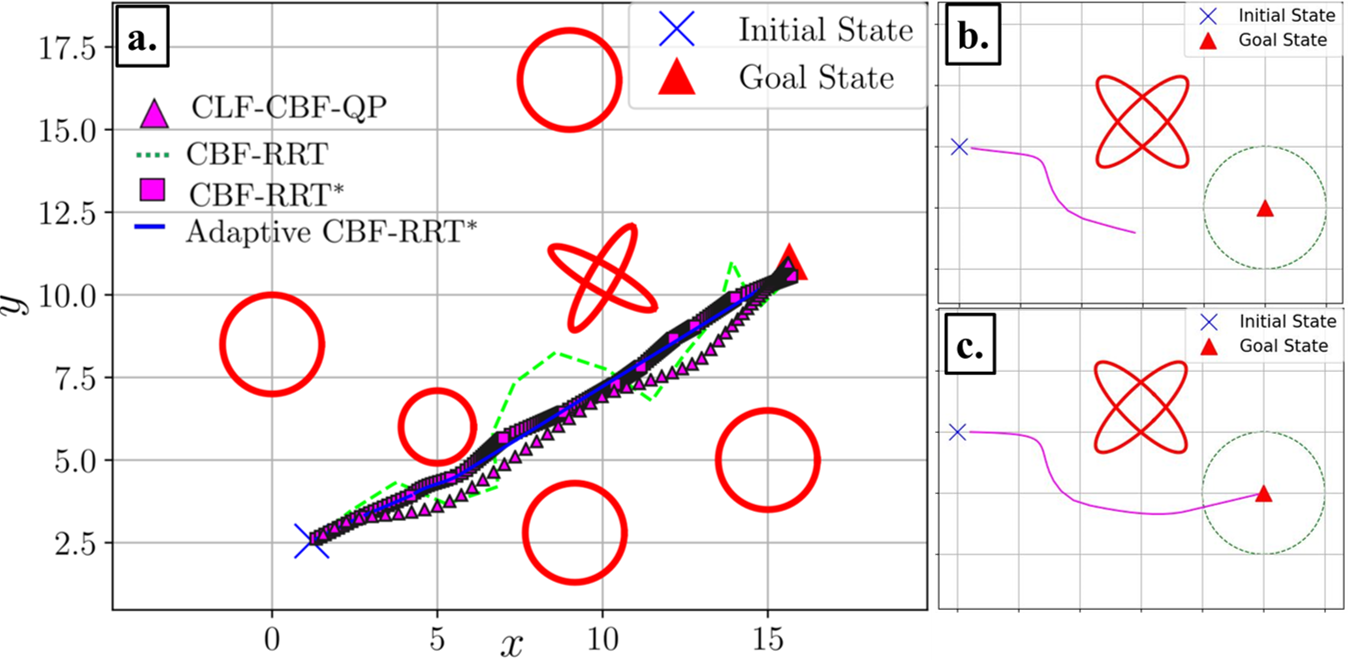}
	\caption{ \textbf{(a)} Multiple motion trajectories for generated using CBF-RRT (green dashed path), CBF-RRT$^\ast$ (magenta boxes path), Adaptive CBF-RRT$^\ast$ (solid blue path), and CLF-CBF-QP exact motion planning (magenta triangles path); \textbf{(b)} and \textbf{(c)} are motion trajectories generated using the CBF-QP based exploratory local motion planner and CBF-CLF-QP exact motion planner, respectively.  }
	\label{fig:sim_all_paths}
\end{figure}

\begin{figure}[htb]
	\begin{minipage}[t]{0.48\linewidth}\centering
		\includegraphics[width=4.5cm]{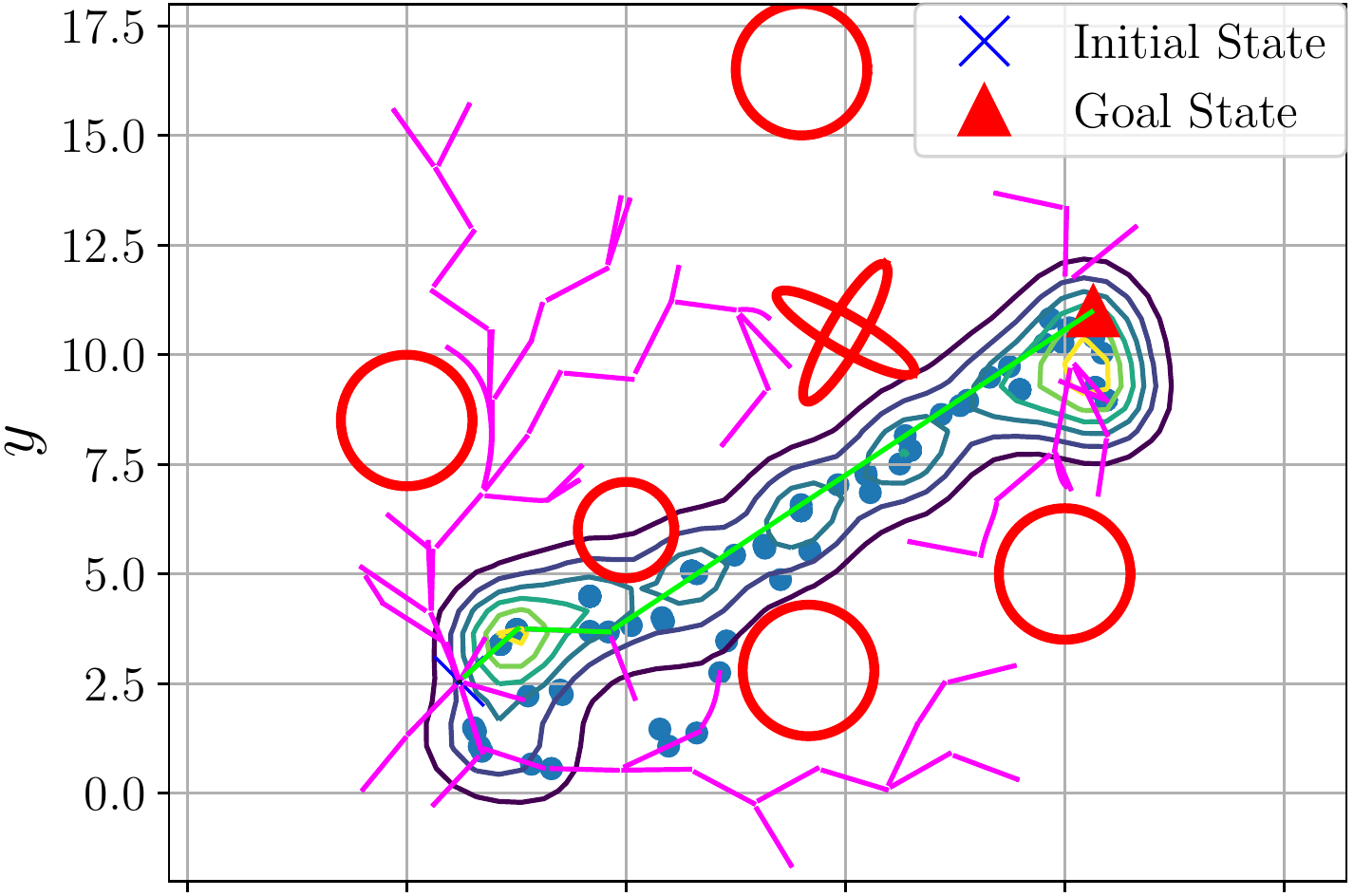}
		\medskip
%		\subcaption[]{}\label{}
	\end{minipage}\hfill
	\begin{minipage}[t]{0.5\linewidth}\centering
		\includegraphics[width=4cm]{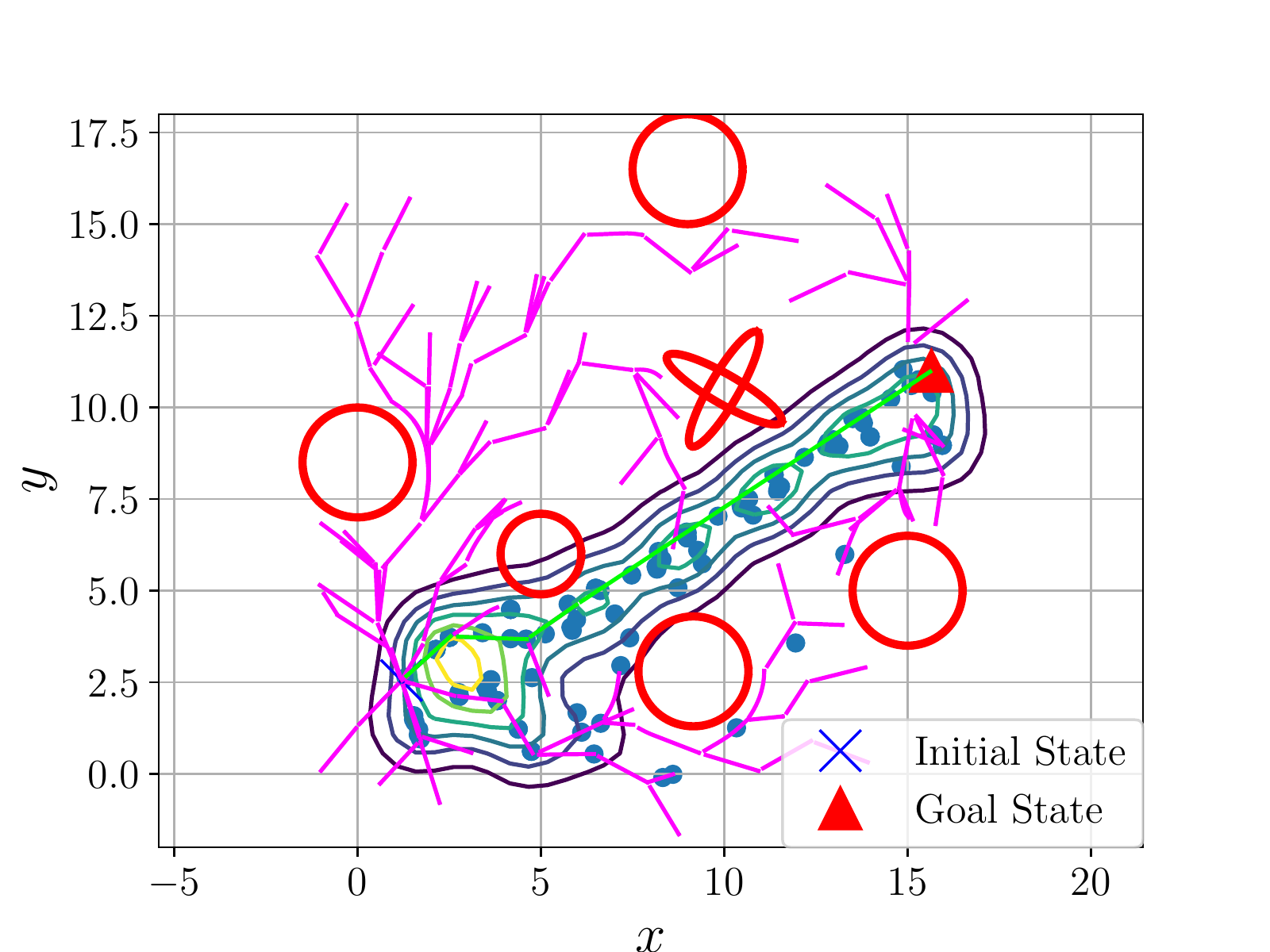}
		\medskip
%		\subcaption[]{}\label{}
	\end{minipage}
	%%%%%%%%%%%%%%%%%%%%%%%%%%%%%%%%%%%%%%%%%%%
	\begin{minipage}[t]{0.49\linewidth}\centering
		\includegraphics[width=4.5cm]{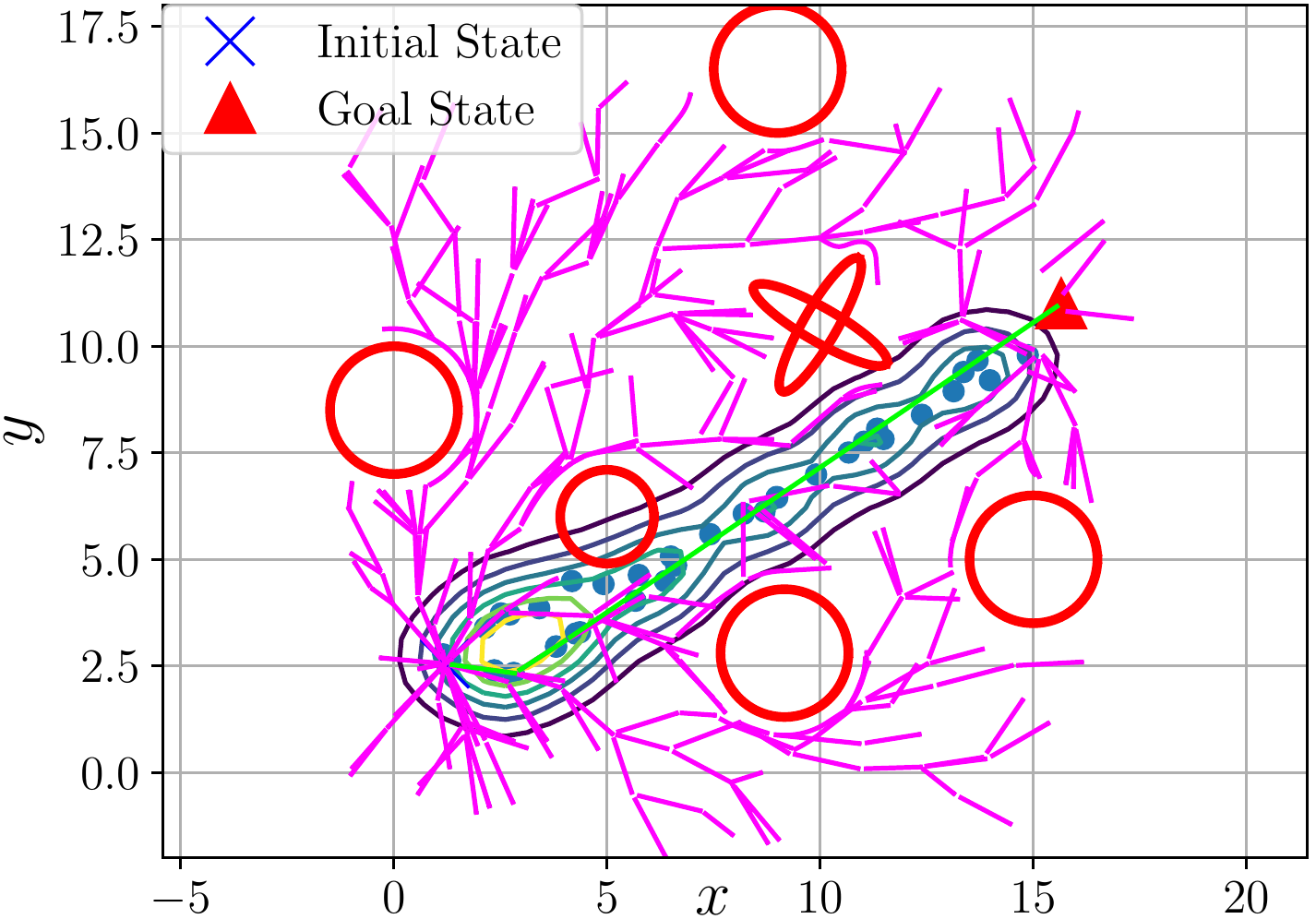}
		\medskip
%		\subcaption[]{}\label{}
	\end{minipage}\hfill
	\begin{minipage}[t]{0.49\linewidth}\centering
		\includegraphics[width=4cm]{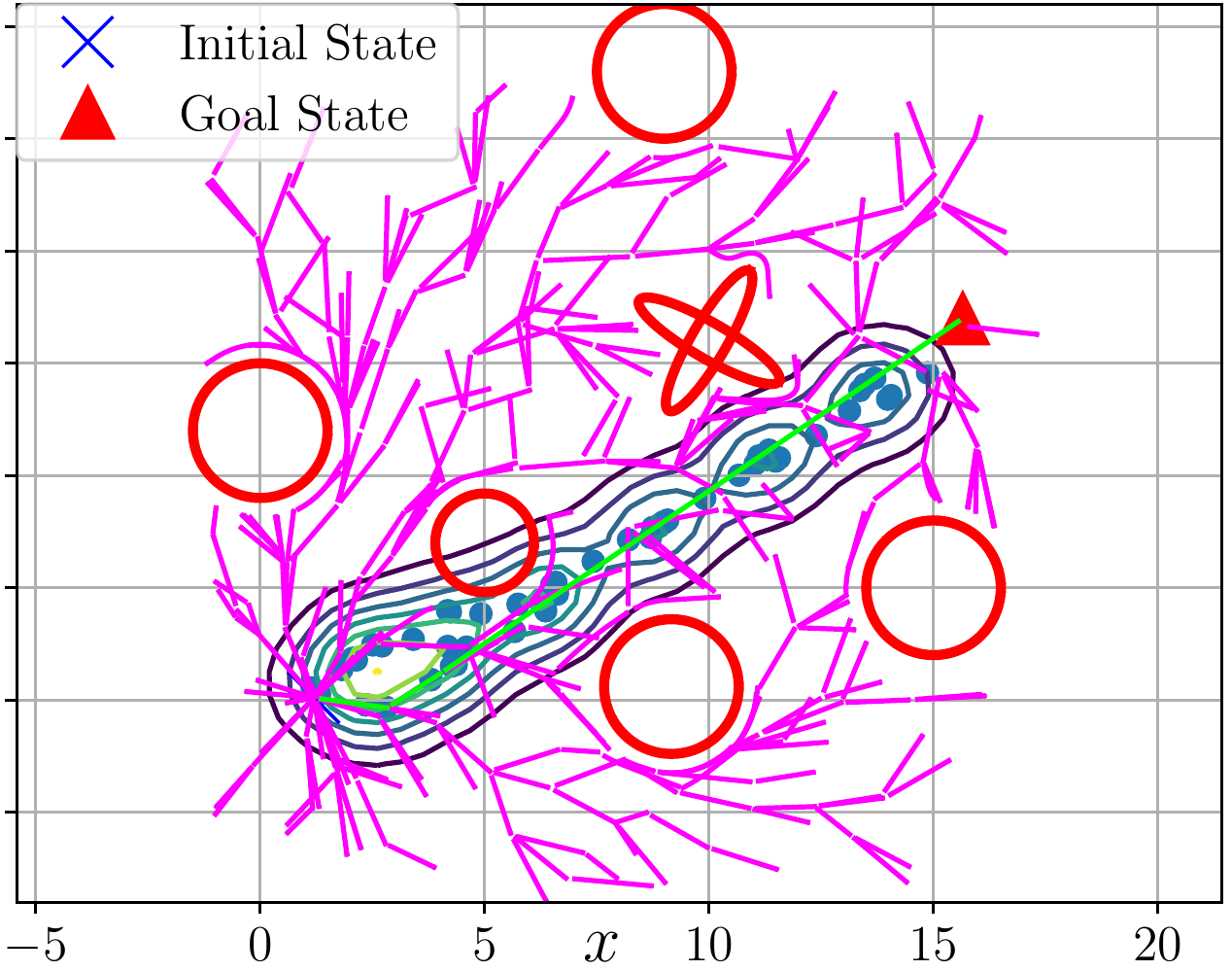}
		\medskip
%		\subcaption[]{}\label{}
	\end{minipage}
	\caption{The evolution of Adaptive CBF-RRT$^\ast$ tree with the elite samples at each CEM iteration. The estimated SDF level sets are shown at each adaptation iteration. In this run the adaptive sampling procedure terminated after 4 iterations, where the K-L divergence between the $3^{rd}$ (bottom right) and $4^{th}$ (bottom left) iterations is $0.06$. }
	\label{fig:Demo_of_adaptiv_CBF-RRT*}
\end{figure}

\begin{figure}[htb]
	\begin{minipage}[t]{0.49\linewidth}\centering
		\includegraphics[width=4.5cm]{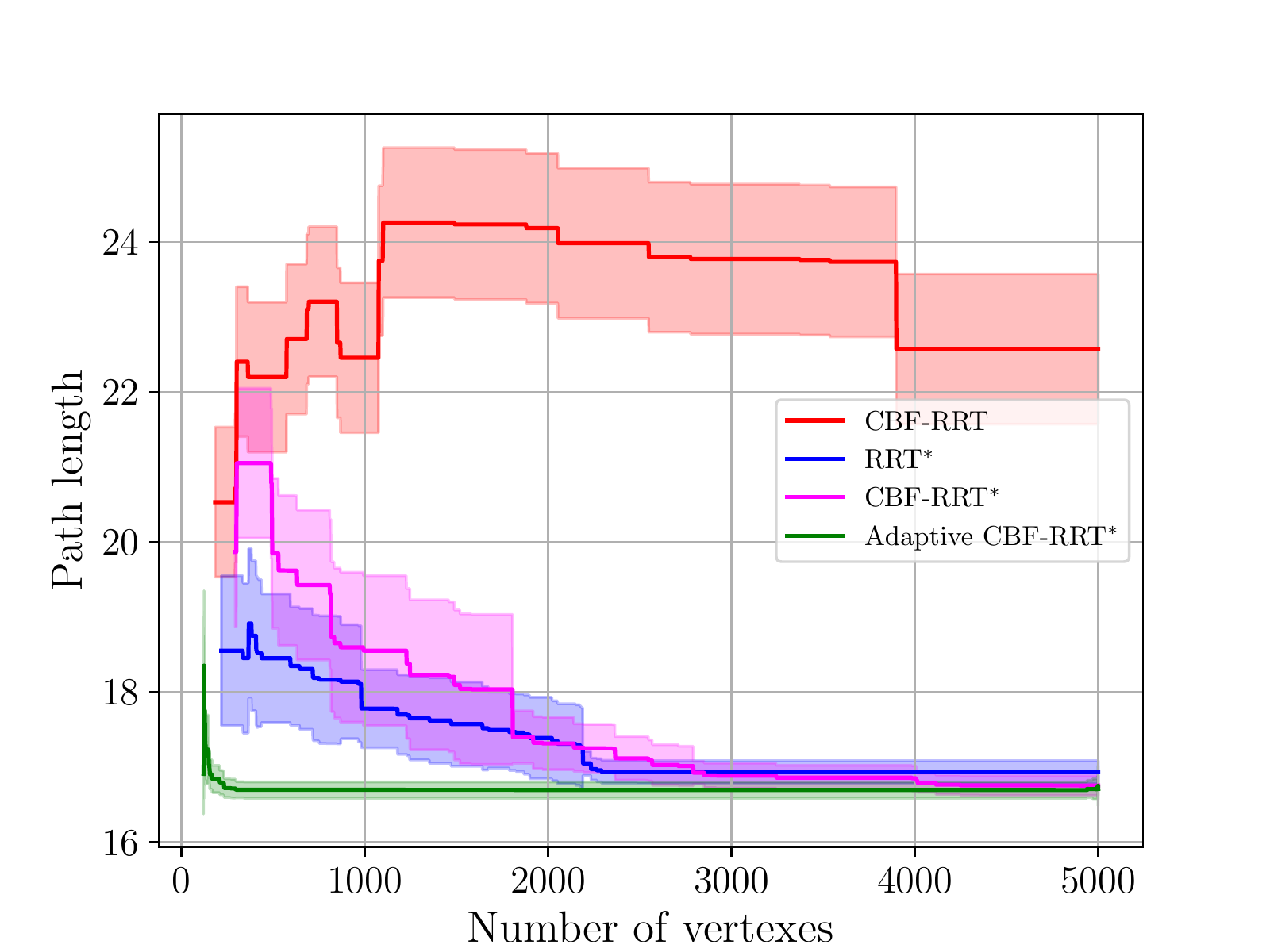}
		\medskip
%		\subcaption[]{}\label{}
	\end{minipage}\hfill
	\begin{minipage}[t]{0.49\linewidth}\centering
		\includegraphics[width=4.5cm]{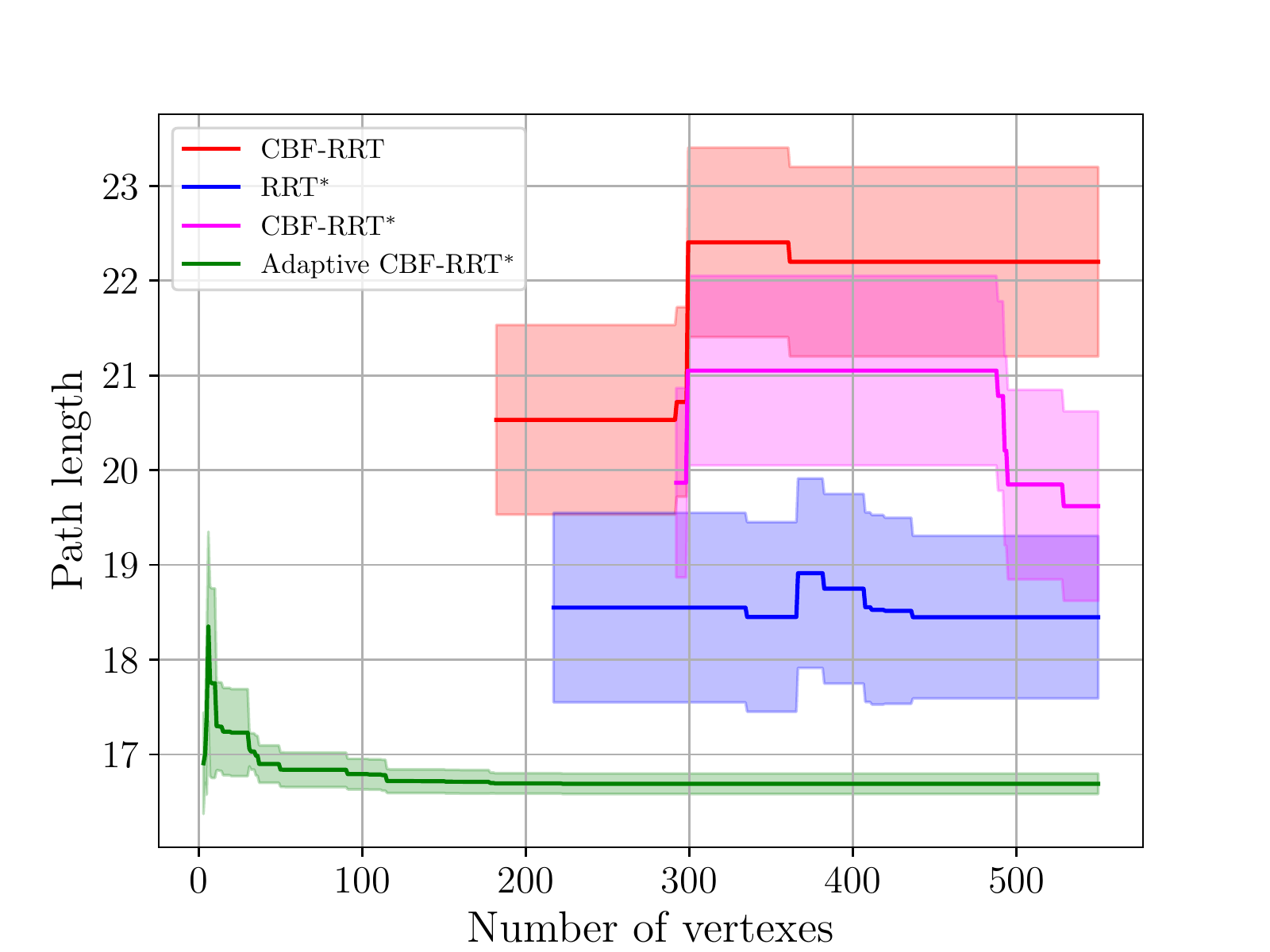}
		\medskip
%		\subcaption[]{}\label{}
	\end{minipage}
	\caption{The average path length of 20 independent runs of
RRT$^\ast$, CBF-RRT, CBF-RRT$^\ast$
, and Adaptive CBF-RRT$^\ast$ with $\%95$
confidence interval. Adaptive CBF-RRT$^\ast$ finds a feasible path to the goal as soon as the \texttt{extToGoal} procedure succeeds to steer to $\mathcal{Q}_{goal}$. }
	\label{fig:20RunsSimulation}
\end{figure}

%SSSSSSSSSSSSSSSSSSSSSSSSSSSSSSSSSSSSSSSSSSSSSSSSSSSSSSSS
\section{Conclusion and Future Work}
%SSSSSSSSSSSSSSSSSSSSSSSSSSSSSSSSSSSSSSSSSSSSSSSSSSSSSSSS
In this work, we introduced two variants of RRT$^\ast$, (Adaptive) CBF-RRT$^\ast$, to approximate a solution for the optimal motion planning problem. Inspired by CBF-RRT \cite{Yang2019a}, we utilized the recent advances in controlling safety-critical systems via CBFs to generate feasible local motion plans that are guaranteed to be collision-free. We prove, under some assumptions, that CBF-RRT$^\ast$ is probabilistically complete. Furthermore, and for efficient exploration, we equip CBF-RRT$^\ast$ with an IS procedure, which is inspired by CE-RRT$^\ast$ \cite{Kobilarov2012e}, and uses CEM algorithm WGKDE to estimate IS density functions. The procedure adapts the SDF of CBF-RRT$^\ast$ to focus the sampling around the optimal solution of the motion planning problem.    

The proposed variants are demonstrated through numerical simulation, and they have been shown to outperform analogous algorithms. 

In this work, we considered sampling in the configuration space of the robot. For future work, we consider extending the presented work to sample in the control inputs space, which might be simpler and could lead to better results. Given that we tested the proposed work to plan for unicycle robots, we consider plan for robots with other dynamics. Furthermore,  the asymptotic optimally of (Adaptive) CBF-RRT$^\ast$ need to be investigated.

% \begin{thebibliography}{99}
% \bibliographystyle{ieee}
% \bibliography{MS_Thesis_biblibrary.bib}
% \end{thebibliography}
\bibliographystyle{IEEEtran}
\bibliography{IEEEabrv,iros_refs}
% \bibliography{IEEEabrv,MS_Thesis_biblibrary}

\end{document}